\documentclass[copyright,creativecommons,noncommercial]{eptcs}  
\providecommand{\event}{SEMSP}

\usepackage{graphicx}
\usepackage{times}
\usepackage{tabulary}
\usepackage{array}  

\usepackage{amsmath,amsthm,amssymb}    
\theoremstyle{definition}  
\newtheorem{theorem}{Theorem}[section]

\newtheorem{definition}[theorem]{Definition}

\newtheorem{example}[theorem]{Example}

\newtheorem{example*}[theorem]{Example*}
\newtheorem{examples*}[theorem]{Examples*}

\newtheorem{remark*}[theorem]{Remark*}

\usepackage{color}
\def\bR{\begin{color}{red}}     
\def\bB{\begin{color}{blue}}
\def\bM{\begin{color}{magenta}}  
\def\bC{\begin{color}{cyan}}  
\def\bW{\begin{color}{white}}
\def\bBl{\begin{color}{black}}   
\def\bG{\begin{color}{green}}  
\def\bY{\begin{color}{yellow}}  
\def\e{\end{color}\xspace}
\providecommand{\urlalt}[2]{\href{#1}{#2}}

\usepackage{docmute}
\usepackage{enumerate,multirow}
\usepackage{amsmath}
\usepackage{amsthm,amsfonts,amssymb,mathtools,xspace}
\usepackage{tikz}
\usetikzlibrary{decorations.pathreplacing}
\usetikzlibrary{decorations.markings}
\usetikzlibrary{calc}
\usetikzlibrary{shapes.geometric}
\usepackage{etoolbox}
\usepackage{datetime}

\theoremstyle{plain}
\newenvironment{pic}[1][]
{\begin{aligned}\begin{tikzpicture}[#1]}
{\end{tikzpicture}\end{aligned}}
\newcommand{\edges}[1][]%
{%\end{scope}\end{pgfonlayer}\begin{pgfonlayer}{foreground}\begin{scope}[#1]
}

\makeatletter
\def\calign@preamble{%
   &\hfil\strut@
    \setboxz@h{\@lign$\m@th\displaystyle{##}$}%
    \ifmeasuring@\savefieldlength@\fi
    \set@field
    \hfil
    \tabskip\alignsep@
}
\let\cmeasure@\measure@
\patchcmd\cmeasure@{\divide\@tempcntb\tw@}{}{}{}
\patchcmd\cmeasure@{\divide\@tempcntb\tw@}{}{}{}
\patchcmd\cmeasure@{\ifodd\maxfields@
  \global\advance\maxfields@\@ne
  \fi}{}{}{}

\makeatother


\newcommand\tinymatrix[1]
{\left( \hspace{-2pt} \renewcommand\thickspace{\kern2pt} \scriptstyle\begin{smallmatrix} #1 \end{smallmatrix} \hspace{-2pt} \right)}

\newcommand\ignore[1]{}

\pgfdeclarelayer{foreground}
\pgfdeclarelayer{background}
\pgfsetlayers{main,foreground,background}
\tikzset{smallbox/.style={draw, fill=white, minimum height=0.45cm, minimum width=0.45cm, inner sep=-100pt}}

\makeatletter
\pgfkeys{%
  /tikz/on layer/.code={
    \pgfonlayer{#1}\begingroup
    \aftergroup\endpgfonlayer
    \aftergroup\endgroup
  },
  /tikz/node on layer/.code={
    \gdef\node@@on@layer{%
      \setbox\tikz@tempbox=\hbox\bgroup\pgfonlayer{#1}\unhbox\tikz@tempbox\endpgfonlayer\egroup}
    \aftergroup\node@on@layer
  },
  /tikz/end node on layer/.code={
    \endpgfonlayer\endgroup\endgroup
  }
}
\def\node@on@layer{\aftergroup\node@@on@layer}
\makeatother\def\thickness{0.7pt}
\tikzstyle{oldmorphism}=[minimum width=30pt, minimum height=16pt, draw, font=\small, inner sep=0pt, fill=white, line width=\thickness]
\tikzstyle{cross}=[preaction={draw=white, -, line width=10pt}]
\tikzstyle{braid}=[double=black, line width=3*\thickness, double distance=\thickness, white]
\tikzstyle{string}=[line width=\thickness]
\tikzstyle{scalar}=[circle, inner sep=0pt, minimum width=15pt, draw, line width=\thickness]
\tikzstyle{dot}=[circle, draw=black, fill=black!25, inner sep=.4ex, line width=\thickness, node on layer=foreground]
\tikzstyle{blackdot}=[circle, draw=black, fill=black!100, inner sep=.4ex, line width=\thickness, node on layer=foreground]
\tikzstyle{graydot}=[circle, draw=black, fill=gray!40!white, inner sep=.4ex, line width=\thickness, node on layer=foreground]
\tikzstyle{whitedot}=[circle, draw=black, fill=white, inner sep=.4ex, line width=\thickness, node on layer=foreground]
\tikzstyle{mixedmorphism}=[morphism, minimum width=30pt, minimum height=16pt, draw, font=\small, inner sep=0pt, fill=white, line width=\thickness,rounded corners=1ex]
\tikzstyle{thick}=[line width=\thickness]
\tikzstyle{tiny}=[font=\tiny]

\tikzset{arrow/.style={decoration={
    markings,
    mark=at position #1 with \arrow{thickarrow}},
    postaction=decorate}
}
\tikzset{reverse arrow/.style={decoration={
    markings,
    mark=at position #1 with \arrow{reversethickarrow}},
    postaction=decorate}
}

\newif\ifblack\pgfkeys{/tikz/black/.is if=black}
\newif\ifwedge\pgfkeys{/tikz/wedge/.is if=wedge}
\newif\ifvflip\pgfkeys{/tikz/vflip/.is if=vflip}
\newif\ifhflip\pgfkeys{/tikz/hflip/.is if=hflip}
\newif\ifhvflip\pgfkeys{/tikz/hvflip/.is if=hvflip}
\newif\ifconnectnw\pgfkeys{/tikz/connect nw/.is if=connectnw}
\newif\ifconnectne\pgfkeys{/tikz/connect ne/.is if=connectne}
\newif\ifconnectsw\pgfkeys{/tikz/connect sw/.is if=connectsw}
\newif\ifconnectse\pgfkeys{/tikz/connect se/.is if=connectse}
\newif\ifconnectn\pgfkeys{/tikz/connect n/.is if=connectn}
\newif\ifconnects\pgfkeys{/tikz/connect s/.is if=connects}
\newif\ifconnectnwf\pgfkeys{/tikz/connect nw >/.is if=connectnwf}
\newif\ifconnectnef\pgfkeys{/tikz/connect ne >/.is if=connectnef}
\newif\ifconnectswf\pgfkeys{/tikz/connect sw >/.is if=connectswf}
\newif\ifconnectsef\pgfkeys{/tikz/connect se >/.is if=connectsef}
\newif\ifconnectnf\pgfkeys{/tikz/connect n >/.is if=connectnf}
\newif\ifconnectsf\pgfkeys{/tikz/connect s >/.is if=connectsf}
\newif\ifconnectnwr\pgfkeys{/tikz/connect nw </.is if=connectnwr}
\newif\ifconnectner\pgfkeys{/tikz/connect ne </.is if=connectner}
\newif\ifconnectswr\pgfkeys{/tikz/connect sw </.is if=connectswr}
\newif\ifconnectser\pgfkeys{/tikz/connect se </.is if=connectser}
\newif\ifconnectnr\pgfkeys{/tikz/connect n </.is if=connectnr}
\newif\ifconnectsr\pgfkeys{/tikz/connect s </.is if=connectsr}
\tikzset{keylengthnw/.initial=\connectheight}
\tikzset{keylengthn/.initial =\connectheight}
\tikzset{keylengthne/.initial=\connectheight}
\tikzset{keylengthsw/.initial=\connectheight}
\tikzset{keylengths/.initial =\connectheight}
\tikzset{keylengthse/.initial=\connectheight}
\tikzset{connect nw length/.style={connect nw=true, keylengthnw={#1}}}
\tikzset{connect n length/.style ={connect n =true, keylengthn ={#1}}}
\tikzset{connect ne length/.style={connect ne=true, keylengthne={#1}}}
\tikzset{connect sw length/.style={connect sw=true, keylengthsw={#1}}}
\tikzset{connect s length/.style ={connect s =true, keylengths ={#1}}}
\tikzset{connect se length/.style={connect se=true, keylengthse={#1}}}
\tikzset{connect nw < length/.style={connect nw <=true, keylengthnw={#1}}}
\tikzset{connect n < length/.style ={connect n <=true,  keylengthn ={#1}}}
\tikzset{connect ne < length/.style={connect ne <=true, keylengthne={#1}}}
\tikzset{connect sw < length/.style={connect sw <=true, keylengthnw={#1}}}
\tikzset{connect s < length/.style ={connect s <=true,  keylengths ={#1}}}
\tikzset{connect se < length/.style={connect se <=true, keylengthse={#1}}}
\tikzset{connect nw > length/.style={connect nw >=true, keylengthnw={#1}}}
\tikzset{connect n > length/.style ={connect n >=true,  keylengthn ={#1}}}
\tikzset{connect ne > length/.style={connect ne >=true, keylengthne={#1}}}
\tikzset{connect sw > length/.style={connect sw >=true, keylengthsw={#1}}}
\tikzset{connect s > length/.style ={connect s >=true,  keylengths ={#1}}}
\tikzset{connect se > length/.style={connect se >=true, keylengthse={#1}}}

\newlength\morphismheight
\newlength\minimummorphismwidth
\newlength\stateheight
\newlength\minimumstatewidth
\newlength\connectheight
\tikzset{width/.initial=\minimummorphismwidth}

\makeatletter
\pgfarrowsdeclare{thickarrow}{thickarrow}
{
  \pgfutil@tempdima=-0.84pt%
  \advance\pgfutil@tempdima by-1.3\pgflinewidth%
  \pgfutil@tempdimb=-1.7pt%
  \advance\pgfutil@tempdimb by.625\pgflinewidth%
  \pgfarrowsleftextend{+\pgfutil@tempdima}
  \pgfarrowsrightextend{+\pgfutil@tempdimb}
}
{
  \pgfmathparse{\pgfgetarrowoptions{thickarrow}}%
  \pgfsetlinewidth{1.25 pt}
  \pgfutil@tempdima=0.28pt%
  \advance\pgfutil@tempdima by.3\pgflinewidth%
  \pgfsetlinewidth{0.8\pgflinewidth}
  \pgfsetdash{}{+0pt}
  \pgfsetroundcap
  \pgfsetroundjoin
  \pgfpathmoveto{\pgfqpoint{-3\pgfutil@tempdima}{4\pgfutil@tempdima}}
  \pgfpathcurveto
  {\pgfqpoint{-2.75\pgfutil@tempdima}{2.5\pgfutil@tempdima}}
  {\pgfqpoint{0pt}{0.25\pgfutil@tempdima}}
  {\pgfqpoint{0.75\pgfutil@tempdima}{0pt}}
  \pgfpathcurveto
  {\pgfqpoint{0pt}{-0.25\pgfutil@tempdima}}
  {\pgfqpoint{-2.75\pgfutil@tempdima}{-2.5\pgfutil@tempdima}}
  {\pgfqpoint{-3\pgfutil@tempdima}{-4\pgfutil@tempdima}}
  \pgfusepathqstroke
}
\pgfarrowsdeclare{reversethickarrow}{reversethickarrow}
{
  \pgfutil@tempdima=-0.84pt%
  \advance\pgfutil@tempdima by-1.3\pgflinewidth%
  \pgfutil@tempdimb=0.2pt%
  \advance\pgfutil@tempdimb by.625\pgflinewidth%
  \pgfarrowsleftextend{+\pgfutil@tempdima}
  \pgfarrowsrightextend{+\pgfutil@tempdimb}
}
{
  \pgftransformxscale{-1}
  \pgfmathparse{\pgfgetarrowoptions{thickarrow}}%
  \ifpgfmathunitsdeclared%
    \pgfmathparse{\pgfmathresult pt}%
  \else%
    \pgfmathparse{\pgfmathresult*\pgflinewidth}%
  \fi%
  \let\thickness=\pgfmathresult
  \pgfsetlinewidth{1.25 pt}
  \pgfutil@tempdima=0.28pt%
  \advance\pgfutil@tempdima by.3\pgflinewidth%
  \pgfsetlinewidth{0.8\pgflinewidth}
  \pgfsetdash{}{+0pt}
  \pgfsetroundcap
  \pgfsetroundjoin
  \pgfpathmoveto{\pgfqpoint{-3\pgfutil@tempdima}{4\pgfutil@tempdima}}
  \pgfpathcurveto
  {\pgfqpoint{-2.75\pgfutil@tempdima}{2.5\pgfutil@tempdima}}
  {\pgfqpoint{0pt}{0.25\pgfutil@tempdima}}
  {\pgfqpoint{0.75\pgfutil@tempdima}{0pt}}
  \pgfpathcurveto
  {\pgfqpoint{0pt}{-0.25\pgfutil@tempdima}}
  {\pgfqpoint{-2.75\pgfutil@tempdima}{-2.5\pgfutil@tempdima}}
  {\pgfqpoint{-3\pgfutil@tempdima}{-4\pgfutil@tempdima}}
  \pgfusepathqstroke
}
\makeatother

\makeatletter
\pgfdeclareshape{ground}
{
    \savedanchor\centerpoint
    {
        \pgf@x=0pt
        \pgf@y=0pt
    }
    \anchor{center}{\centerpoint}
    \anchorborder{\centerpoint}
    \anchor{north}
    {
        \pgf@x=0pt
        \pgf@y=0.5*0.33*\stateheight
    }
    \anchor{south}
    {
        \pgf@x=0pt
        \pgf@y=0pt
    }
    \saveddimen\overallwidth
    {
        \pgfkeysgetvalue{/pgf/minimum width}{\minwidth}
        \pgf@x=\minimumstatewidth
        \ifdim\pgf@x<\minwidth
            \pgf@x=\minwidth
        \fi
    }
    \backgroundpath
    {
        \begin{pgfonlayer}{foreground}
        \pgfsetstrokecolor{black}
        \pgfsetlinewidth{1.25pt}
        \ifhflip
            \pgftransformyscale{-1}
        \fi
        \pgftransformscale{0.5}
        \pgfpathmoveto{\pgfpoint{-0.5*\overallwidth}{0}}
        \pgfpathlineto{\pgfpoint{0.5*\overallwidth}{0}}
        \pgfpathmoveto{\pgfpoint{-0.33*\overallwidth}{0.33*\stateheight}}
        \pgfpathlineto{\pgfpoint{0.33*\overallwidth}{0.33*\stateheight}}
        \pgfpathmoveto{\pgfpoint{-0.16*\overallwidth}{0.66*\stateheight}}
        \pgfpathlineto{\pgfpoint{0.16*\overallwidth}{0.66*\stateheight}}
        \pgfpathmoveto{\pgfpoint{-0.02*\overallwidth}{\stateheight}}
        \pgfpathlineto{\pgfpoint{0.02*\overallwidth}{\stateheight}}
        \pgfusepath{stroke}
        \end{pgfonlayer}
    }
}
\tikzset{forward arrow style/.style={every to/.style, decoration={
    markings,
    mark=at position 0.5 with \arrow{thickarrow}},
    postaction=decorate}}
\tikzset{reverse arrow style/.style={every to/.style, decoration={
    markings,
    mark=at position 0.5 with \arrow{reversethickarrow}},
    postaction=decorate}}
\pgfdeclareshape{morphism}
{
    \savedanchor\centerpoint
    {
        \pgf@x=0pt
        \pgf@y=0pt
    }
    \anchor{center}{\centerpoint}
    \anchorborder{\centerpoint}
    \saveddimen\savedlengthnw
    {
        \pgfkeysgetvalue{/tikz/keylengthnw}{\len}
        \pgf@x=\len
    }
    \saveddimen\savedlengthn
    {
        \pgfkeysgetvalue{/tikz/keylengthn}{\len}
        \pgf@x=\len
    }
    \saveddimen\savedlengthne
    {
        \pgfkeysgetvalue{/tikz/keylengthne}{\len}
        \pgf@x=\len
    }
    \saveddimen\savedlengthsw
    {
        \pgfkeysgetvalue{/tikz/keylengthsw}{\len}
        \pgf@x=\len
    }
    \saveddimen\savedlengths
    {
        \pgfkeysgetvalue{/tikz/keylengths}{\len}
        \pgf@x=\len
    }
    \saveddimen\savedlengthse
    {
        \pgfkeysgetvalue{/tikz/keylengthse}{\len}
        \pgf@x=\len
    }
    \saveddimen\overallwidth
    {
        \pgfkeysgetvalue{/tikz/width}{\minwidth}
        \pgf@x=\wd\pgfnodeparttextbox
        \ifdim\pgf@x<\minwidth
            \pgf@x=\minwidth
        \fi
    }
    \savedanchor{\upperrightcorner}
    {
        \pgf@y=.5\ht\pgfnodeparttextbox
        \advance\pgf@y by -.5\dp\pgfnodeparttextbox
        \pgf@x=.5\wd\pgfnodeparttextbox
    }
    \anchor{north}
    {
        \pgf@x=0pt
        \pgf@y=0.5\morphismheight
    }
    \anchor{north east}
    {
        \pgf@x=\overallwidth
        \multiply \pgf@x by 2
        \divide \pgf@x by 5
        \pgf@y=0.5\morphismheight
    }
    \anchor{east}
    {
        \pgf@x=\overallwidth
        \divide \pgf@x by 2
        \advance \pgf@x by 5pt
        \pgf@y=0pt
    }
    \anchor{west}
    {
        \pgf@x=-\overallwidth
        \divide \pgf@x by 2
        \advance \pgf@x by -5pt
        \pgf@y=0pt
    }
    \anchor{north west}
    {
        \pgf@x=-\overallwidth
        \multiply \pgf@x by 2
        \divide \pgf@x by 5
        \pgf@y=0.5\morphismheight
    }
    \anchor{connect nw}
    {
        \pgf@x=-\overallwidth
        \multiply \pgf@x by 2
        \divide \pgf@x by 5
        \pgf@y=0.5\morphismheight
        \advance\pgf@y by \savedlengthnw
    }
    \anchor{connect ne}
    {
        \pgf@x=\overallwidth
        \multiply \pgf@x by 2
        \divide \pgf@x by 5
        \pgf@y=0.5\morphismheight
        \advance\pgf@y by \savedlengthne
    }
    \anchor{connect sw}
    {
        \pgf@x=-\overallwidth
        \multiply \pgf@x by 2
        \divide \pgf@x by 5
        \pgf@y=-0.5\morphismheight
        \advance\pgf@y by -\savedlengthsw
    }
    \anchor{connect se}
    {
        \pgf@x=\overallwidth
        \multiply \pgf@x by 2
        \divide \pgf@x by 5
        \pgf@y=-0.5\morphismheight
        \advance\pgf@y by -\savedlengthse
    }
    \anchor{connect n}
    {
        \pgf@x=0pt
        \pgf@y=0.5\morphismheight
        \advance\pgf@y by \savedlengthn
    }
    \anchor{connect s}
    {
        \pgf@x=0pt
        \pgf@y=-0.5\morphismheight
        \advance\pgf@y by -\savedlengths
    }
    \anchor{south east}
    {
        \pgf@x=\overallwidth
        \multiply \pgf@x by 2
        \divide \pgf@x by 5
        \pgf@y=-0.5\morphismheight
    }
    \anchor{south west}
    {
        \pgf@x=-\overallwidth
        \multiply \pgf@x by 2
        \divide \pgf@x by 5
        \pgf@y=-0.5\morphismheight
    }
    \anchor{south}
    {
        \pgf@x=0pt
        \pgf@y=-0.5\morphismheight
    }
    \anchor{text}
    {
        \upperrightcorner
        \pgf@x=-\pgf@x
        \pgf@y=-\pgf@y
    }
    \backgroundpath
    {
        \pgfsetstrokecolor{black}
        \pgfsetlinewidth{\thickness}
        \begin{scope}
                \ifhflip
                    \pgftransformyscale{-1}
                \fi
                \ifvflip
                    \pgftransformxscale{-1}
                \fi
                \ifhvflip
                    \pgftransformxscale{-1}
                    \pgftransformyscale{-1}
                \fi
                \pgfpathmoveto{\pgfpoint
                    {-0.5*\overallwidth-5pt}
                    {0.5*\morphismheight}}
                \pgfpathlineto{\pgfpoint
                    {0.5*\overallwidth+5pt}
                    {0.5*\morphismheight}}
                \ifwedge
                    \pgfpathlineto{\pgfpoint
                        {0.5*\overallwidth + 15pt}
                        {-0.5*\morphismheight}}
                \else
                    \pgfpathlineto{\pgfpoint
                        {0.5*\overallwidth + 5pt}
                        {-0.5*\morphismheight}}
                \fi
                \pgfpathlineto{\pgfpoint
                    {-0.5*\overallwidth-5pt}
                    {-0.5*\morphismheight}}
                \pgfpathclose
                \pgfusepath{stroke}
        \end{scope}
        \ifconnectnw
            \pgfpathmoveto{\pgfpoint
                {-0.4*\overallwidth}
                {0.5*\morphismheight}}
            \pgfpathlineto{\pgfpoint
                {-0.4*\overallwidth}
                {0.5*\morphismheight+\savedlengthnw}}
            \pgfusepath{stroke}
        \fi
        \ifconnectne
            \pgfpathmoveto{\pgfpoint
                {0.4*\overallwidth}
                {0.5*\morphismheight}}
            \pgfpathlineto{\pgfpoint
                {0.4*\overallwidth}
                {0.5*\morphismheight+\savedlengthne}}
            \pgfusepath{stroke}
        \fi
        \ifconnectsw
            \pgfpathmoveto{\pgfpoint
                {-0.4*\overallwidth}
                {-0.5*\morphismheight}}
            \pgfpathlineto{\pgfpoint
                {-0.4*\overallwidth}
                {-0.5*\morphismheight-\savedlengthsw}}
            \pgfusepath{stroke}
        \fi
        \ifconnectse
            \pgfpathmoveto{\pgfpoint
                {0.4*\overallwidth}
                {-0.5*\morphismheight}}
            \pgfpathlineto{\pgfpoint
                {0.4*\overallwidth}
                {-0.5*\morphismheight-\savedlengthse}}
            \pgfusepath{stroke}
        \fi
        \ifconnectn
            \pgfpathmoveto{\pgfpoint
                {0pt}
                {0.5*\morphismheight}}
            \pgfpathlineto{\pgfpoint
                {0pt}
                {0.5*\morphismheight+\savedlengthn}}
            \pgfusepath{stroke}
        \fi
        \ifconnects
            \pgfpathmoveto{\pgfpoint
                {0pt}
                {-0.5*\morphismheight}}
            \pgfpathlineto{\pgfpoint
                {0pt}
                {-0.5*\morphismheight-\savedlengths}}
            \pgfusepath{stroke}
        \fi
        \ifconnectnwf
            \draw [forward arrow style] (-0.4*\overallwidth,0.5*\morphismheight)
                to (-0.4*\overallwidth,0.5*\morphismheight+\savedlengthnw);
        \fi
        \ifconnectnef
            \draw [forward arrow style] (0.4*\overallwidth,0.5*\morphismheight)
                to (0.4*\overallwidth,0.5*\morphismheight+\savedlengthne);
        \fi
        \ifconnectswf
            \draw [forward arrow style] (-0.4*\overallwidth,-0.5*\morphismheight-\savedlengthsw)
                to (-0.4*\overallwidth,-0.5*\morphismheight);
        \fi
        \ifconnectsef
            \draw [forward arrow style] (0.4*\overallwidth,-0.5*\morphismheight-\savedlengthse)
                to (0.4*\overallwidth,-0.5*\morphismheight);
        \fi
        \ifconnectnf
            \draw [forward arrow style] (0,0.5*\morphismheight)
                to (0,0.5*\morphismheight+\savedlengthn);
        \fi
        \ifconnectsf
            \draw [forward arrow style] (0,-0.5*\morphismheight-\savedlengths)
                to (0,-0.5*\morphismheight);
        \fi
        \ifconnectnwr
            \draw [reverse arrow style] (-0.4*\overallwidth,0.5*\morphismheight)
                to (-0.4*\overallwidth,0.5*\morphismheight+\savedlengthnw);
        \fi
        \ifconnectner
            \draw [reverse arrow style] (0.4*\overallwidth,0.5*\morphismheight)
                to (0.4*\overallwidth,0.5*\morphismheight+\savedlengthne);
        \fi
        \ifconnectswr
            \draw [reverse arrow style] (-0.4*\overallwidth,-0.5*\morphismheight-\savedlengthsw)
                to (-0.4*\overallwidth,-0.5*\morphismheight);
        \fi
        \ifconnectser
            \draw [reverse arrow style] (0.4*\overallwidth,-0.5*\morphismheight-\savedlengthse)
                to (0.4*\overallwidth,-0.5*\morphismheight);
        \fi
        \ifconnectnr
            \draw [reverse arrow style] (0,0.5*\morphismheight)
                to (0,0.5*\morphismheight+\savedlengthn);
        \fi
        \ifconnectsr
            \draw [reverse arrow style] (0,-0.5*\morphismheight-\savedlengths)
                to (0,-0.5*\morphismheight);
        \fi
    }
}
\pgfdeclareshape{swish right}
{
    \savedanchor\centerpoint
    {
        \pgf@x=0pt
        \pgf@y=0pt
    }
    \anchor{center}{\centerpoint}
    \anchorborder{\centerpoint}
    \anchor{north}
    {
        \pgf@x=\minimummorphismwidth
        \divide\pgf@x by 5
        \pgf@y=\morphismheight
        \divide\pgf@y by 2
        \advance\pgf@y by \connectheight
    }
    \anchor{south}
    {
        \pgf@x=-\minimummorphismwidth
        \divide\pgf@x by 5
        \pgf@y=-\morphismheight
        \divide\pgf@y by 2
        \advance\pgf@y by -\connectheight
    }
    \backgroundpath
    {
        \pgfsetstrokecolor{black}
        \pgfsetlinewidth{\thickness}
        \pgfpathmoveto{\pgfpoint
            {-0.2*\minimummorphismwidth}
            {-0.5*\morphismheight-\connectheight}}
        \pgfpathcurveto
            {\pgfpoint{-0.2*\minimummorphismwidth}{0pt}}
            {\pgfpoint{0.2*\minimummorphismwidth}{0pt}}
            {\pgfpoint
                {0.2*\minimummorphismwidth}
                {0.5*\morphismheight+\connectheight}}
        \pgfusepath{stroke}
    }
}
\pgfdeclareshape{swish left}
{
    \savedanchor\centerpoint
    {
        \pgf@x=0pt
        \pgf@y=0pt
    }
    \anchor{center}{\centerpoint}
    \anchorborder{\centerpoint}
    \anchor{north}
    {
        \pgf@x=-\minimummorphismwidth
        \divide\pgf@x by 5
        \pgf@y=\morphismheight
        \divide\pgf@y by 2
        \advance\pgf@y by \connectheight
    }
    \anchor{south}
    {
        \pgf@x=\minimummorphismwidth
        \divide\pgf@x by 5
        \pgf@y=-\morphismheight
        \divide\pgf@y by 2
        \advance\pgf@y by -\connectheight
    }
    \backgroundpath
    {
        \pgfsetstrokecolor{black}
        \pgfsetlinewidth{\thickness}
        \pgfpathmoveto{\pgfpoint
            {0.2*\minimummorphismwidth}
            {-0.5*\morphismheight-\connectheight}}
        \pgfpathcurveto
            {\pgfpoint{0.2*\minimummorphismwidth}{0pt}}
            {\pgfpoint{-0.2*\minimummorphismwidth}{0pt}}
            {\pgfpoint
                {-0.2*\minimummorphismwidth}
                {0.5*\morphismheight+\connectheight}}
        \pgfusepath{stroke}
    }
}
\pgfdeclareshape{state}
{
    \savedanchor\centerpoint
    {
        \pgf@x=0pt
        \pgf@y=0pt
    }
    \anchor{center}{\centerpoint}
    \anchorborder{\centerpoint}
    \saveddimen\overallwidth
    {
        \pgf@x=3\wd\pgfnodeparttextbox
        \ifdim\pgf@x<\minimumstatewidth
            \pgf@x=\minimumstatewidth
        \fi
    }
    \savedanchor{\upperrightcorner}
    {
        \pgf@x=.5\wd\pgfnodeparttextbox
        \pgf@y=.5\ht\pgfnodeparttextbox
        \advance\pgf@y by -.5\dp\pgfnodeparttextbox
    }
    \anchor{A}
    {
        \pgf@x=-\overallwidth
        \divide\pgf@x by 4
        \pgf@y=0pt
    }
    \anchor{B}
    {
        \pgf@x=\overallwidth
        \divide\pgf@x by 4
        \pgf@y=0pt
    }
    \anchor{text}
    {
        \upperrightcorner
        \pgf@x=-\pgf@x
        \ifhflip
            \pgf@y=-\pgf@y
            \advance\pgf@y by 0.4\stateheight
        \else
            \pgf@y=-\pgf@y
            \advance\pgf@y by -0.4\stateheight
        \fi
    }
    \backgroundpath
    {
        \pgfsetstrokecolor{black}
        \pgfsetlinewidth{\thickness}
        \pgfpathmoveto{\pgfpoint{-0.5*\overallwidth}{0}}
        \pgfpathlineto{\pgfpoint{0.5*\overallwidth}{0}}
        \ifhflip
            \pgfpathlineto{\pgfpoint{0}{\stateheight}}
        \else
            \pgfpathlineto{\pgfpoint{0}{-\stateheight}}
        \fi
        \pgfpathclose
        \ifblack
            \pgfsetfillcolor{black!50}
            \pgfusepath{fill,stroke}
        \else
            \pgfusepath{stroke}
        \fi
    }
}
\makeatother

\newcommand{\tinyhandle}[1][dot]{\raisebox{-2pt}{\ensuremath{\hspace{-3pt}\begin{pic}[scale=0.4,string]
        \node (0) at (0,0) {};
        \node[dot, inner sep=1.0pt] (1) at (0,0.3) {};
        \node[dot, inner sep=1.0pt] (2) at (0,0.7) {};
        \node (3) at (0,1) {};
        \draw (0.center) to (1.center);
        \draw (2.center) to (3.center);
        \draw[in=180, out=180, looseness=2] (1.center) to (2.center);
        \draw[in=0, out=0, looseness=2] (1.center) to (2.center);
\end{pic}\hspace{-1pt}}}}

\newcommand\ket[1]{\ensuremath{| #1 \rangle}}

\usepackage{color}

\ignore{
\usepackage{mathtools}
\tikzstyle{dot}=[inner sep=0.7mm,minimum width=0pt,minimum
height=0pt,fill=black,draw=black,shape=circle]
\tikzstyle{black dot}=[dot,fill=black]
\tikzstyle{white dot}=[dot,fill=white]
\tikzstyle{gray dot}=[dot,fill=gray!40!white]
}

\title{Quantum Algorithms for\\Compositional Natural Language Processing}
\author{William Zeng and Bob Coecke\thanks{This work was funded by Air Force Office of Scientific Research grant FA9550-14-1-0079 and The Rhodes Trust.}
  \institute{Rigetti Computing and University of Oxford}
  \email{zeng.will@gmail.com - bob.coecke@cs.ox.ac.uk}    
  }

\def\titlerunning{Quantum Algorithms for Compositional Natural Language Processing}     
\def\authorrunning{W.~Zeng \& B.~Coecke}     

\begin{document}
  
\maketitle

\begin{abstract}
        We propose a new application of quantum computing to the field of natural language processing.  Ongoing work in this field attempts to incorporate grammatical structure into algorithms that compute meaning.  In \cite{clark2008compositional, coecke2010mathematical}, Coecke, Sadrzadeh and Clark introduce such a model (the CSC model) based on tensor product composition. While this algorithm has many advantages, its implementation is hampered by the large classical computational resources that it requires. In this work we show how computational shortcomings of the  CSC approach could be resolved using quantum computation (possibly in addition to existing techniques for dimension reduction). We address the value of quantum RAM \cite{giovannetti2008quantum} for this model and extend an algorithm from Wiebe, Braun and Lloyd \cite{wiebe2014quantum} into a quantum algorithm to categorize  sentences in CSC. Our new algorithm demonstrates a quadratic speedup over classical methods under certain conditions.
%\keywords{quantum algorithms \and natural language processing \and compositional semantics \and vector space models of meaning}
\end{abstract}

\section{Introduction} 

As human computer interfaces become more advanced, natural language processing has grown to be a ubiquitous part of our world.  Its techniques allow computers to understand natural language to perform tasks like automatic summarization, machine translation, information retrieval, and sentiment analysis. Most approaches to this problem, such as Google's search, understand strings of separate words in a `bag of words' approach, ignoring any grammatical structure. This is certainly unsatisfactory, as we know that the meaning of a sentence is more than the meaning of its component words. Research in \textit{distributional compositional semantics} (DisCo) seeks to address this by incorporating grammatical structure. 

In \cite{clark2008compositional, coecke2010mathematical}, Coecke, Sadrzadeh and Clark introduce a DisCo model (the CSC model) based on tensor product composition that gives a grammatically informed algorithm to compute the meaning of sentences and phrases.  Interestingly, this algorithm was directly inspired by quantum theory, and more specifically, teleportation-like protocols \cite{teleling}. While this algorithm has many advantages, its implementation is hampered by the large classical computational resources that it requires.  This paper presents ways that quantum computers potentially could solve some of these problems, and doing so without the further assumptions and inaccuracies of the existing classical techniques for dimension reduction (e.g.~those based on \cite{plate1991holographic}). Hence making the CSC an attractive application for quantum computation.

We use the fact that quantum computation is naturally suited to managing high dimensional tensor product spaces. Recent literature has shown that quantum algorithms can thus provide advantages for machine learning \cite{wiebe2014quantum,rebentrost2014quantum}, inference \cite{low2014quantum}, and regression \cite{wiebe2012quantum,wang2014quantum} tasks.  We leverage these results in two particular ways:
\begin{enumerate}
\item We employ the scaling of quantum systems to address computational difficulties inherent in tensor-product based compositional semantics.
\item Shared structure makes algorithms in the CSC model especially amenable to quantum speedups.  We specify a CSC sentence similarity algorithm that, under certain conditions, gives quadratic speedups for natural language tasks.  
\end{enumerate}

In Section \ref{sec:disco}, we cover the basic framework of distributional compositional linguistics. Section 3 introduces the advantages of quantum representations for this framework.  Sections 4 and 5 propose a quantum algorithm with quadratic speedup for calculating sentence similarity within CSC. Section 6 briefly discusses the noise tolerance of these methods.

\section{Distributional Compositional Semantics and the CSS model}
\label{sec:disco}
In modern natural language processing, the \emph{vector space model} is widely used to compute the meaning of individual words \cite{schutze1998automatic}. In this approach we first specify a set of context words, for example the 2000 most common words in a given corpus.  These context words then form the basis of the vector space of word meanings in the following manner: for some given word, say  ``quantum", we look through a corpus and count the frequency with which each basis word    appears `near' to ``quantum". It is likely that we would have a high frequency for words like ``physics" and ``information" for example.  These frequencies then form the \emph{word vector} for ``quantum". Words represented by normalised vectors are similar if the inner product of their word vectors is close to one. These `bag of words' methods are typically referred to as \emph{distributional}.

As the same sentence rarely occurs repeatedly, this distributional technique cannot be directly extended to calculate the meaning of longer phrases, sentences, paragraphs, etc. Instead, \emph{compositional} semantics designs algorithms for deriving the meaning of a sentence or phrase from known meanings of component words,  taking into account types and grammatical structure \cite{lambek2008word}. The \emph{distributional compositional} semantic model (DisCo) combines both approaches to introduce grammatical form to the composition of word vectors \cite{clark2008compositional, coecke2010mathematical}. 

%See Fig. 1. 
%
%\begin{figure}[ht]
%\label{fig:discofig}
%\begin{center}
%\begin{tabular}{ccc}
% Distributional & & Compositional \\
% \includegraphics[width=3cm]{distributional.png} &\quad\raisebox{1.6cm}{\mbox{{\color{red} \huge +}}} \qquad\quad& 
% \raisebox{1cm}{\includegraphics[width=3cm]{compositional.png}}
%\end{tabular}
%\end{center}
%\caption{The DisCo approach combines word vectors with pregroup or combinatory categorical grammar. The diagram on the right shows which terms cancel in the derivation tree.  It is drawn suggestively as explained in Section \ref{sec:disco}. }
%\end{figure}

In this model, each grammatical type is assigned a tensor product space based on Lambek's pregroup grammar \cite{lambek2008word} or combinatorial categorical grammar \cite{hermann2013role}. The meaning of nouns is, for example, calculated as in the distributional case, and we label their vector space $\mathcal{N}$.  A transitive verb's meaning is then, following the grammar, a vector in the space $\mathcal{N}\otimes \mathcal{S} \otimes \mathcal{N}$, where $\mathcal{S}$ is the meaning space for sentences \cite{clark2008compositional, coecke2010mathematical}. An intuition for this is that the transitive verb takes a subject noun as a left argument and an object noun as a right argument. An adjective lives in the space $\mathcal{N}\otimes\mathcal{N}$.

We use a diagrammatic notation for vectors, tensors, and linear maps as is common for both CSC and quantum information. Here vertical composition (read top to bottom) represents composition of linear maps and horizontal composition represents the tensor product:
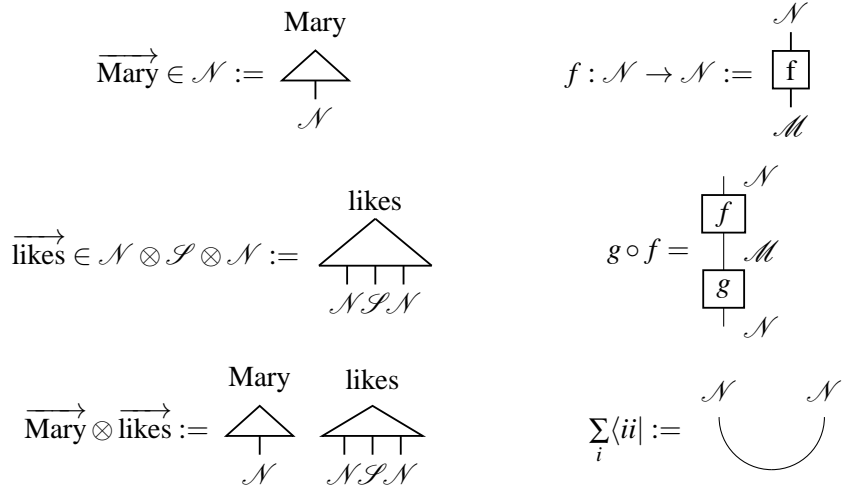
\[
\begin{tabular}{cc}
$\overrightarrow{\mbox{Mary}}\in \mathcal{N} :=    
\begin{aligned}
\begin{tikzpicture}[scale=0.5]
                \node  (20) at (0, 0) {$\mathcal{N}$};
                \node  (31) at (0, 2.5) {Mary};
                \node (1) [state, hflip] at (0,1) {};
                \draw [style = thick]  (0,0.5) to (1.center);
\end{tikzpicture}
\end{aligned}$ & \qquad
$f:\mathcal{N}\to\mathcal{N} :=  
\begin{aligned}
\begin{tikzpicture}[scale=0.5]
                \node  (20) at (0, -1.5) {$\mathcal{M}$};
                \node  (20) at (0, 1.5) {$\mathcal{N}$};                
                \node (1) [morphism] at (0,0) {f};
                \draw [style = thick]  (0,1) to (0,0.5);   
                \draw [style = thick]  (0,-1) to (0,-0.5);
\end{tikzpicture}
\end{aligned}$ \\
$\overrightarrow{\mbox{likes}}\in \mathcal{N}\otimes\mathcal{S}\otimes \mathcal{N} :=  
\begin{aligned}
\begin{tikzpicture}[scale=0.5]
                \node  (3) at (-1, 1) {};
                \node  (4) at (2, 1) {};
                \node  (5) at (0.5, 2.25) {};
                \node  (10) at (-0.25, 1) {};
                \node  (11) at (0.5, 1) {};
                \node  (12) at (1.25, 1) {};
                \node  (15) at (-0.25, 0.5) {};
                \node  (16) at (0.5, 0.5) {};
                \node  (17) at (1.25, 0.5) {};
                \node  (20) at (-0.25, 0) {$\mathcal{N}$};
                \node  (21) at (0.5, 0) {$\mathcal{S}$};
                \node  (22) at (1.25, 0) {$\mathcal{N}$};    
                \node  (31) at (0.5, 2.75) {likes};
                \draw [style = thick]  (3.center) to (4.center);
                \draw [style = thick] (4.center) to (5.center);
                \draw [style = thick] (3.center) to (5.center);
                \draw [style = thick] (10.center) to (15.center);
                \draw [style = thick] (11.center) to (16.center);
                \draw [style = thick] (12.center) to (17.center);
\end{tikzpicture}
\end{aligned}$ 
 & \qquad
$ \begin{aligned}
 \begin{tikzpicture}[yscale=1]
                \node (0) at (0, 1) {};
                \node (2) at (-1, 0) {$g\circ f = $};
                \node (1) at (0.5, 1) {$\mathcal{N}$};
                \node (3) at (0, -1) {};
                \node [style=morphism] (4) at (0, -0.5) {$g$};
                \node (5) at (0.5, -1) {$\mathcal{N}$};
                \node (6) at (0.5, 0) {$\mathcal{M}$};
                \node [style=morphism] (7) at (0, 0.5) {$f$};
                \draw (3.center) to (0, -0.75);
                \draw (0, -0.25) to (0, 0.25);
                \draw (0, 0.75) to (0.center);
    \end{tikzpicture}
    \end{aligned}$ \\
$\overrightarrow{\mbox{Mary}}\otimes \overrightarrow{\mbox{likes}} :=
\begin{aligned}
\begin{tikzpicture}[scale=0.5]
                \node  (20) at (0, 0) {$\mathcal{N}$};
                \node  (31) at (0, 2.5) {Mary};
                \node (1) [state, hflip] at (0,1) {};
                \draw [style = thick]  (0,0.5) to (1.center);
                \node (2) [state, hflip, xscale=1.5, yscale=1] at (3,1) {};
                \node  (32) at (3, 2.5) {likes};
                \node  (20) at (2.25, 0) {$\mathcal{N}$};                                \node  (20) at (3, 0) {$\mathcal{S}$};
                \node  (20) at (3.75, 0) {$\mathcal{N}$};
                \draw [style = thick]  (3,0.5) to (2.center);
                \draw [style = thick]  (3.75,0.5) to (3.75,1);                            \draw [style = thick]  (2.25,0.5) to (2.25,1);    
\end{tikzpicture}
\end{aligned} $ & \qquad\qquad
$\sum\limits_i \langle ii| :=
\begin{aligned}
\begin{tikzpicture}[scale=0.7]
                \node  (0) at (-1, 1.5) {$\mathcal{N}$};
                \node  (1) at (1, 1.5) {$\mathcal{N}$};
                \node  (0) at (-1, 1) {};
                \node  (1) at (1, 1) {};
                \node  (2) at (0, 0) {};
                \draw [bend right=45, looseness=1.00] (0.center) to (2.center);  
                \draw [bend right=45, looseness=1.00] (2.center) to (1.center);               
\end{tikzpicture}  
\end{aligned}$
\end{tabular}
\]
where $f:\mathcal{N}\to\mathcal{M}$ and $g:\mathcal{M}\to\mathcal{N}$ are linear maps and the linear map $\sum_i\langle ii|$ sums over all the basis vectors of $\mathcal{N}$ and is called a \emph{cap}.

Given a well-typed sentence with meaning vectors $\vec{w_j}$ for each of $k$ words, the classical CSC algorithm for calculating a sentence's meaning is \cite{teleling}:  
\begin{enumerate}
\item Compute the tensor product $\overrightarrow{\mbox{words}}=\vec{w_0}\otimes...\otimes \vec{w_k}$ in the order that each word appears in the sentence.

\item Construct a linear map that represents the grammatical type reduction by ``wiring up" the vectors with the appropriate caps.  For example, given that the pregroup type reduction of a noun/transitive verb/noun sentence is:
\[
n \cdot {}^{-1}n \cdot s \cdot n^{-1} \cdot n\leq 1\cdot s\cdot 1 \leq s
\]
 this linear map is:
\[
\sum\limits_i \langle ii| \ \, \otimes \ {\rm id} \ \otimes\ \, \sum\limits_i \langle ii|
\]
\item Compute the meaning of the sentence by applying the linear map to the $\overrightarrow{\mbox{words}}$ vector. This results in a vector in $\mathcal{S}$ which corresponds to the meaning of the sentence.  In the above example the result corresponds to the diagram:
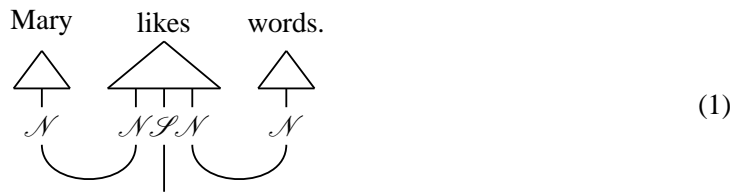
\pgfdeclarelayer{edgelayer}
\pgfdeclarelayer{nodelayer}
\pgfsetlayers{background,edgelayer,nodelayer,main}
\tikzstyle{none}=[inner sep=0mm]
\tikzstyle{every loop}=[]
\tikzstyle{mark coordinate}=[inner sep=0pt,outer sep=0pt,minimum size=3pt,fill=black,circle]

\begin{equation}
\label{eq:svo}
\begin{aligned}
\begin{tikzpicture}[scale=0.5]
        \begin{pgfonlayer}{nodelayer}     
                \node [style=none] (0) at (-3.5, 1) {};
                \node [style=none] (1) at (-2, 1) {};
                \node [style=none] (2) at (-2.75, 2) {};
                \node [style=none] (3) at (-1, 1) {};
                \node [style=none] (4) at (2, 1) {};
                \node [style=none] (5) at (0.5, 2.25) {};
                \node [style=none] (6) at (4.5, 1) {};
                \node [style=none] (7) at (3.75, 2) {};
                \node [style=none] (8) at (3, 1) {};
                \node [style=none] (9) at (-2.75, 1) {};
                \node [style=none] (10) at (-0.25, 1) {};
                \node [style=none] (11) at (0.5, 1) {};
                \node [style=none] (12) at (1.25, 1) {};
                \node [style=none] (13) at (3.75, 1) {};
                \node [style=none] (14) at (-2.75, 0.5) {};
                \node [style=none] (15) at (-0.25, 0.5) {};
                \node [style=none] (16) at (0.5, 0.5) {};
                \node [style=none] (17) at (1.25, 0.5) {};
                \node [style=none] (18) at (3.75, 0.5) {};
                \node [style=none] (19) at (-2.75, 0) {$\mathcal{N}$};
                \node [style=none] (20) at (-0.25, 0) {$\mathcal{N}$};
                \node [style=none] (21) at (0.5, 0) {$\mathcal{S}$};
                \node [style=none] (22) at (1.25, 0) {$\mathcal{N}$};
                \node [style=none] (23) at (3.75, 0) {$\mathcal{N}$};
                \node [style=none] (24) at (-2.75, -0.5) {};
                \node [style=none] (25) at (-0.25, -0.5) {};
                \node [style=none] (26) at (0.5, -0.5) {};
                \node [style=none] (27) at (1.25, -0.5) {};
                \node [style=none] (28) at (3.75, -0.5) {};
                \node [style=none] (29) at (0.5, -1.75) {};
                \node [style=none] (30) at (-2.75, 2.75) {Mary};
                \node [style=none] (31) at (0.5, 2.75) {likes};
                \node [style=none] (32) at (3.75, 2.75) {words.};
        \end{pgfonlayer}
        \begin{pgfonlayer}{edgelayer}
                \draw [style = thick] (0.center) to (1.center);
                \draw [style = thick] (1.center) to (2.center);
                \draw [style = thick] (2.center) to (0.center);
                \draw [style = thick]  (3.center) to (4.center);
                \draw [style = thick] (8.center) to (6.center);
                \draw [style = thick] (6.center) to (7.center);
                \draw [style = thick] (4.center) to (5.center);
                \draw [style = thick] (3.center) to (5.center);
                \draw [style = thick] (8.center) to (7.center);
                \draw [style = thick] (9.center) to (14.center);
                \draw [style = thick] (10.center) to (15.center);
                \draw [style = thick] (11.center) to (16.center);
                \draw [style = thick] (12.center) to (17.center);
                \draw [style = thick] (13.center) to (18.center);
                \draw [thick, bend right=90, looseness=1.25] (24.center) to (25.center);
                \draw [thick, bend right=90, looseness=1.25] (27.center) to (28.center);
                \draw [style = thick] (26.center) to (29.center);
        \end{pgfonlayer}
\end{tikzpicture}
\end{aligned}
\end{equation}
\end{enumerate}

\noindent We refer the reader to \cite{coecke2010mathematical} for a fuller description of the distributional compositional model and to \cite{experimental-catcompdist, KartSadr} and \cite{kartsaklis2012unified} for experimental implementations.

These models suggest a promising approach to incorporate grammatical structure with vector space models of meaning, yet they come with the computational challenges of large tensor product spaces \cite{GrefenstetteThesis2013}. While there do exist some classical approaches to avoid the calculation of the full tensor product, such as the holographic reduced representations from \cite{plate1991holographic} or the use of dimensionality reduction \cite{polajnar2013learning}, these approaches introduce further assumptions and inaccuracies.  For this reason, ameliorating the large computational costs introduced these large spaces through quantum computation is of particular interest.

\section{Quantum computation for the CSC model}

The most immediate advantage for quantum implementations of the CSC model is gained by storing meaning vectors in quantum systems.  For $\alpha,\beta\in \mathbb{C}$ a two-level quantum system, a qubit, is defined by:
\begin{center}
  \begin{tabular}{cc}
   Qubit & Qubits  \\
  $ \begin{array}{lcl}
        \ket{\psi} &=& \alpha\ket{0}+\beta\ket{1} \\[0.2cm]
        &=&\alpha\left(\begin{array}{c} 1 \\ 0 \end{array}\right)
            +\beta\left(\begin{array}{c} 0 \\ 1 \end{array}\right)
   \end{array} $
   &\qquad\qquad
  $ \ket{\psi_1}\otimes\ket{\psi_2} = \left(\begin{array}{c} \alpha_1\alpha_2 \\ \alpha_1\beta_2\\\beta_1\alpha_2\\\beta_2\beta_2 \end{array}\right)$
   \\  
  \end{tabular}
\end{center}
and composition of qubits is given by the tensor product.  This leaves each $n$-qubit system with $2^n$ degrees of freedom, indicating that $N$-dimensional classical vectors can be stored in $\log_2 N$ qubits.
Consider a corpus whose word-meaning space is given by a basis of the 2,000 most common words. Even if we make the simplifying assumption that the sentence-meaning space is no larger than the word-meaning space we obtain the dramatic improvements details in Table 1.

\begin{table}[ht]
\label{tab:space}
\begin{center}
\begin{tabular}{|c|c|c|}\hline
 & One Transitive Verb & 10k tr. verbs \\\hline
 Classical & $8\times 10^{9}$ bits & $8\times 10^{13}$ bits \\\hline
 Quantum & 33 qubits & 47 qubits \\\hline
\end{tabular}
\end{center}
\caption{Rough comparisons of the storage necessary for verbs in quantum and classical frameworks.}
\end{table}

Further, these word meanings can be imported into a ``bucket bridgade" quantum RAM that allows them to retrieved with a complexity linear in the number of qubits \cite{giovannetti2008quantum}. The general point is that because quantum systems compose via the tensor product they are natural choices to store complex types and sentences that have the same compositional structure. We can then employ quantum algorithms on for natural language classification as presented in Section \ref{sec:discoQalg}.

\section{A quantum algorithm for the closest vector problem}
\label{sec:qalg}

Many tasks in computational linguistics such as clustering, text classification, phrase/word similarity, and sentiment analysis rely on computations that determine the closest vector to $\vec{s}$ out of some set of $N$-dimensional vectors $\{\vec{v}_0,\vec{v}_1,...\vec{v}_{M-1}\}$. In clustering algorithms, for example, the set of vectors could be either the centroids of different clusters or the full set of training vectors, as in nearest neighbor clustering algorithms. Either the inner-product distance or Euclidean distance can be used. We will assume that all vectors are $N$-dimensional.

\begin{definition}
Given vector $\vec{s}$ and a set of $M$ vectors $U = \{\vec{v}_0,\vec{v}_1,...\vec{v}_{M-1}\}$ the \emph{closest vector problem} asks one to determine which $v_j$ has the smallest inner product distance with $\vec{s}$.
\end{definition}

Direct calculation of the smallest vector would have complexity $\mathcal{O}(MN)$.  In \cite{wiebe2014quantum} the authors introduce a quantum algorithm for this problem that, under certain conditions, demonstrate quadratic speedups over direct calculation and polynomial speedups over Monte-Carlo methods. Some proof of principle experiments have then demonstrated clustering of eight-dimensional vectors, based on these techniques, on a small photonic quantum computer \cite{cai2015entanglement}. This algorithm requires the following assumptions, where we write $v_{ji}$ for the $i^{\tiny\mbox{th}}$ entry of the $j^{\tiny\mbox{th}}$ vector:
\begin{enumerate}
\item Vectors $\vec{v_j}$ and $\vec{s}$ are $d$-sparse, with no more than $d$ non-zero entries.    
\item The relevant vectors are stored in some kind of quantum memory so that the quantum computer can access their entries with the two oracles of the form:
\begin{equation}
\begin{array}{c}
\mathcal{O}\ket{j}\ket{i}\ket{0}:=\ket{j}\ket{i}\ket{v_{ji}}, \\
\mathcal{F}\ket{j}\ket{l}:= \ket{j}\ket{f(j,l)},
\end{array}
\end{equation}
where $f(j,l)$ is the location of the $l^{\tiny\mbox{th}}$ non-zero entry of $v_j$.  It is against these memory access oracles that the performance of our algorithm will be measured.

\item $\max(|v_{ji}|^2)\le r_{\small\mbox{max}}$ for some known constant $r_{\small\mbox{max}}$.  

\item All the vectors are normalized.
 
\end{enumerate}

Under these assumptions we are able to run a quantum nearest-neighbor algorithm with complexity characterized by the following theorem:

\begin{theorem}[\cite{wiebe2014quantum}]
We can find $\max_j|\langle s~|~v_j\rangle|^2$ with success probability $1-\delta$ and error $\epsilon$ using an expected number of $\mathcal{O}$ and $\mathcal{F}$ queries that is bounded above by
\begin{equation}
1080\sqrt{M}\left\lceil\frac{4\pi(\pi+1)d^2r^4_{\tiny\mbox{max}}}{\epsilon}\right\rceil\left\lceil \frac{\ln\left(81M(\ln(M)+\gamma)\right)/\delta_0}{2(8/\pi^2-1/2)^2}\right\rceil,
\end{equation}
where $\gamma\approx 0.5772$ is Euler's constant.
\end{theorem}

It is clear that for this quantum algorithm there is a quadratic improvement in scaling with $M$, the number of training vectors. While the dimension of the vectors $N$ is not explicitly included, in general it is implicitly there through the dependence on $d$.  It is also clear that if $r_{\small\mbox{max}}\propto 1/\sqrt{d}$, then the algorithm's dependence on both $d$ and $N$ drops out. As the vectors are normalized, this can be expected to occur if the vectors have sparsity that grows linearly with their size \cite{wiebe2014quantum}. The authors further assume that for ``typical" cases the error $\epsilon$ scales as $\Theta(1/\sqrt{N})$ so that the runtime for the quantum inner-product algorithm becomes $\mathcal{O}\left(\sqrt{NM}\ln(M)d^2r^4_{\small\mbox{max}}\right)$.\footnote{
This is argued for in Appendix D of \cite{wiebe2014quantum} for a ``typical" case where the vectors are uniformly distributed over the unit sphere.} This result shows a quadratic improvement over direct calculations and also shows improvement over Monte Carlo methods, whose complexity is $\mathcal{O}\left(NMd^2r^4_{\small\mbox{max}}\right)$. These comparisons are summarized in Table 2.

\begin{table}[ht]
\label{tab:comp}
\begin{center}
\begin{tabular}{|c|c|c|}\hline
 Type & Typical cases & Atypical cases    \\ \hline  
 Classical Direct & $\mathcal{O}(NM)$ & $\mathcal{O}(NM)$ \\
 Classical Monte Carlo & $\mathcal{O}\left(NMd^2r_{\small\mbox{max}}^4\right)$ & $\mathcal{O}\left(Md^2r_{\small\mbox{max}}^4/\epsilon^2\right)$ \\
 Quantum & $\mathcal{O}\left(\sqrt{NM}\log(M)d^2r^4_{\small\mbox{max}}\right)$ &  $\mathcal{O}\left(\sqrt{M}\log(M)d^2r^4_{\small\mbox{max}}/\epsilon\right)$ \\\hline
\end{tabular}
\end{center}
\caption{Complexity comparisons for different closest vector algorithms. Adapted from \cite{wiebe2014quantum}.}
\end{table}

 In the following section we adapt this algorithm to sentence similarity calculations in the distributional compositional framework.

\section{A quantum algorithm for CSC sentence similarity}
\label{sec:discoQalg}

The quantum algorithm from Section \ref{sec:qalg} can be used to advantage in natural language processing tasks however, the computational overhead of the CSC approach would dwarf this algorithm's advantages if it were naively applied.  
Throughout this section we will assume both that $r_{\small\mbox{max}}\propto 1/\sqrt{d}$ and that the accuracy necessary for our calculation means $\epsilon$ scales as $\Theta(1/\sqrt{N})$. Consider the example of probabilistically classifying the meaning of a  simple sentence. We illustrate this example with a noun-verb-noun sentence. The meaning of the nouns are vectors in an $N$-dimensional vector space and the meaning of the verb is a vector in the space $\mathcal{N}\otimes \mathcal{S} \otimes \mathcal{N}$. We represent a derivation of the meaning of the full sentence with the following map:

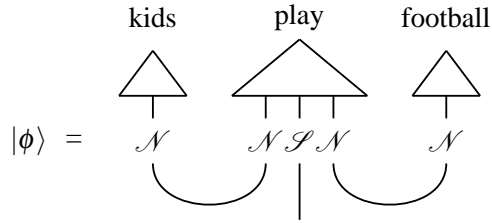
\pgfdeclarelayer{edgelayer}
\pgfdeclarelayer{nodelayer}
\pgfsetlayers{background,edgelayer,nodelayer,main}
\tikzstyle{none}=[inner sep=0mm]
\tikzstyle{every loop}=[]
\tikzstyle{mark coordinate}=[inner sep=0pt,outer sep=0pt,minimum size=3pt,fill=black,circle]  

\[
\label{eqn:phi}
\begin{aligned}
\begin{tikzpicture}[scale=0.6]
        \begin{pgfonlayer}{nodelayer}
                \node [style=none] (0) at (-5.5, 0) {\ket{\phi}};
                \node [style=none] (0) at (-4.5, 0) {=};       
                \node [style=none] (0) at (-3.5, 1) {};
                \node [style=none] (1) at (-2, 1) {};
                \node [style=none] (2) at (-2.75, 2) {};
                \node [style=none] (3) at (-1, 1) {};
                \node [style=none] (4) at (2, 1) {};
                \node [style=none] (5) at (0.5, 2.25) {};
                \node [style=none] (6) at (4.5, 1) {};
                \node [style=none] (7) at (3.75, 2) {};
                \node [style=none] (8) at (3, 1) {};
                \node [style=none] (9) at (-2.75, 1) {};
                \node [style=none] (10) at (-0.25, 1) {};
                \node [style=none] (11) at (0.5, 1) {};
                \node [style=none] (12) at (1.25, 1) {};
                \node [style=none] (13) at (3.75, 1) {};
                \node [style=none] (14) at (-2.75, 0.5) {};
                \node [style=none] (15) at (-0.25, 0.5) {};
                \node [style=none] (16) at (0.5, 0.5) {};
                \node [style=none] (17) at (1.25, 0.5) {};
                \node [style=none] (18) at (3.75, 0.5) {};
                \node [style=none] (19) at (-2.75, 0) {$\mathcal{N}$};
                \node [style=none] (20) at (-0.25, 0) {$\mathcal{N}$};
                \node [style=none] (21) at (0.5, 0) {$\mathcal{S}$};
                \node [style=none] (22) at (1.25, 0) {$\mathcal{N}$};
                \node [style=none] (23) at (3.75, 0) {$\mathcal{N}$};
                \node [style=none] (24) at (-2.75, -0.5) {};
                \node [style=none] (25) at (-0.25, -0.5) {};
                \node [style=none] (26) at (0.5, -0.5) {};
                \node [style=none] (27) at (1.25, -0.5) {};
                \node [style=none] (28) at (3.75, -0.5) {};
                \node [style=none] (29) at (0.5, -1.75) {};
                \node [style=none] (30) at (-2.75, 2.75) {kids};
                \node [style=none] (31) at (0.5, 2.75) {play};
                \node [style=none] (32) at (3.75, 2.75) {football};
        \end{pgfonlayer}
        \begin{pgfonlayer}{edgelayer}
                \draw [style = thick] (0.center) to (1.center);
                \draw [style = thick] (1.center) to (2.center);
                \draw [style = thick] (2.center) to (0.center);
                \draw [style = thick]  (3.center) to (4.center);
                \draw [style = thick] (8.center) to (6.center);
                \draw [style = thick] (6.center) to (7.center);
                \draw [style = thick] (4.center) to (5.center);  
                \draw [style = thick] (3.center) to (5.center);
                \draw [style = thick] (8.center) to (7.center);
                \draw [style = thick] (9.center) to (14.center);
                \draw [style = thick] (10.center) to (15.center);
                \draw [style = thick] (11.center) to (16.center);
                \draw [style = thick] (12.center) to (17.center);
                \draw [style = thick] (13.center) to (18.center);
                \draw [thick, bend right=90, looseness=1.25] (24.center) to (25.center);
                \draw [thick, bend right=90, looseness=1.25] (27.center) to (28.center);
                \draw [style = thick] (26.center) to (29.center);
        \end{pgfonlayer}
\end{tikzpicture}
\end{aligned}
\]

From now on, we will take sentences to exist in the same meaning space as words, i.e. $\mathcal{S}\simeq \mathcal{N}$.

\begin{definition}
For meaning vector $\vec{s}$ and $M$ sets of meaning vectors, a \emph{classification task} assigns $\vec{s}$ to the set containing the nearest-neighbor of $\vec{s}$.
\end{definition}

An example task would be to determine if a sentence is about sports or politics or if a sentence expresses agreement or disagreement. 
%The \emph{classification} approach considers the closest set to be given %by the set whose average vector has the smallest inner product with $s$. %We define these average vectors as 
%\begin{align*}
%v &= \frac{1}{m}\sum_m v_i \qquad \qquad
%w = \frac{1}{m}\sum_m w_i.
%\end{align}
%In the simplest case, there is only a single meaning vector in each training %cluster, i.e. $m=1$. 
If, to present a simplified example, we take each cluster to only contain a single vector ($\vec{v}$ and $\vec{w}$ respectively) then the sentence would be classified by computing 
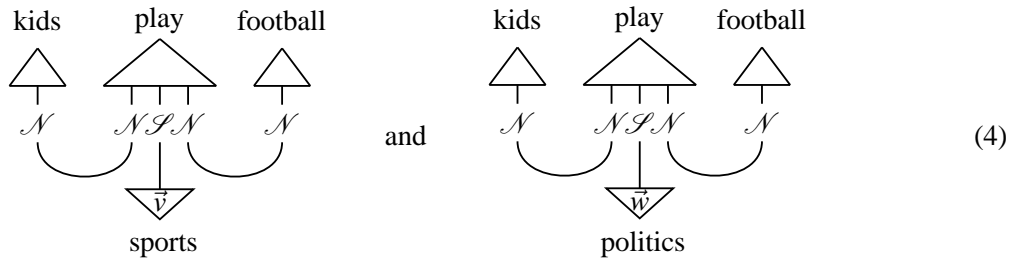
\begin{equation}
\label{eq:exampleCalcs}
\begin{aligned}
\begin{tikzpicture}[scale=0.5]
        \begin{pgfonlayer}{nodelayer}
                \node [style=none] (0) at (-3.5, 1) {};
                \node [style=none] (1) at (-2, 1) {};
                \node [style=none] (2) at (-2.75, 2) {};
                \node [style=none] (3) at (-1, 1) {};
                \node [style=none] (4) at (2, 1) {};
                \node [style=none] (5) at (0.5, 2.25) {};
                \node [style=none] (6) at (4.5, 1) {};
                \node [style=none] (7) at (3.75, 2) {};
                \node [style=none] (8) at (3, 1) {};
                \node [style=none] (9) at (-2.75, 1) {};
                \node [style=none] (10) at (-0.25, 1) {};
                \node [style=none] (11) at (0.5, 1) {};
                \node [style=none] (12) at (1.25, 1) {};
                \node [style=none] (13) at (3.75, 1) {};
                \node [style=none] (14) at (-2.75, 0.5) {};
                \node [style=none] (15) at (-0.25, 0.5) {};
                \node [style=none] (16) at (0.5, 0.5) {};
                \node [style=none] (17) at (1.25, 0.5) {};
                \node [style=none] (18) at (3.75, 0.5) {};
                \node [style=none] (19) at (-2.75, 0) {$\mathcal{N}$};
                \node [style=none] (20) at (-0.25, 0) {$\mathcal{N}$};
                \node [style=none] (21) at (0.5, 0) {$\mathcal{S}$};
                \node [style=none] (22) at (1.25, 0) {$\mathcal{N}$};
                \node [style=none] (23) at (3.75, 0) {$\mathcal{N}$};
                \node [style=none] (24) at (-2.75, -0.5) {};
                \node [style=none] (25) at (-0.25, -0.5) {};
                \node [style=none] (26) at (0.5, -0.5) {};
                \node [style=none] (27) at (1.25, -0.5) {};
                \node [style=none] (28) at (3.75, -0.5) {};
                \node [style=none] (29) at (0.5, -1.75) {};
                \node [style=none] (30) at (-2.75, 2.75) {kids};
                \node [style=none] (31) at (0.5, 2.75) {play};
                \node [style=none] (32) at (3.75, 2.75) {football};
                \node [style=state] (33) at (0.5, -1.75) {$\vec{v}$};
                \node [style=none] (34) at (0.6, -3.25) {sports};
        \end{pgfonlayer}
        \begin{pgfonlayer}{edgelayer}
                \draw [style = thick] (0.center) to (1.center);
                \draw [style = thick] (1.center) to (2.center);
                \draw [style = thick] (2.center) to (0.center);
                \draw [style = thick]  (3.center) to (4.center);
                \draw [style = thick] (8.center) to (6.center);
                \draw [style = thick] (6.center) to (7.center);
                \draw [style = thick] (4.center) to (5.center);
                \draw [style = thick] (3.center) to (5.center);
                \draw [style = thick] (8.center) to (7.center);
                \draw [style = thick] (9.center) to (14.center);
                \draw [style = thick] (10.center) to (15.center);
                \draw [style = thick] (11.center) to (16.center);
                \draw [style = thick] (12.center) to (17.center);
                \draw [style = thick] (13.center) to (18.center);
                \draw [thick, bend right=90, looseness=1.25] (24.center) to (25.center);
                \draw [thick, bend right=90, looseness=1.25] (27.center) to (28.center);
                \draw [style = thick] (26.center) to (29.center);
        \end{pgfonlayer}
\end{tikzpicture}
\end{aligned}  
 \qquad \mbox{and} \qquad %%%%  
\begin{aligned}
\begin{tikzpicture}[scale=0.5]
        \begin{pgfonlayer}{nodelayer}
                \node [style=none] (0) at (-3.5, 1) {};
                \node [style=none] (1) at (-2, 1) {};
                \node [style=none] (2) at (-2.75, 2) {};
                \node [style=none] (3) at (-1, 1) {};
                \node [style=none] (4) at (2, 1) {};
                \node [style=none] (5) at (0.5, 2.25) {};
                \node [style=none] (6) at (4.5, 1) {};
                \node [style=none] (7) at (3.75, 2) {};
                \node [style=none] (8) at (3, 1) {};
                \node [style=none] (9) at (-2.75, 1) {};
                \node [style=none] (10) at (-0.25, 1) {};
                \node [style=none] (11) at (0.5, 1) {};
                \node [style=none] (12) at (1.25, 1) {};
                \node [style=none] (13) at (3.75, 1) {};
                \node [style=none] (14) at (-2.75, 0.5) {};
                \node [style=none] (15) at (-0.25, 0.5) {};
                \node [style=none] (16) at (0.5, 0.5) {};
                \node [style=none] (17) at (1.25, 0.5) {};
                \node [style=none] (18) at (3.75, 0.5) {};
                \node [style=none] (19) at (-2.75, 0) {$\mathcal{N}$};
                \node [style=none] (20) at (-0.25, 0) {$\mathcal{N}$};
                \node [style=none] (21) at (0.5, 0) {$\mathcal{S}$};
                \node [style=none] (22) at (1.25, 0) {$\mathcal{N}$};
                \node [style=none] (23) at (3.75, 0) {$\mathcal{N}$};
                \node [style=none] (24) at (-2.75, -0.5) {};
                \node [style=none] (25) at (-0.25, -0.5) {};
                \node [style=none] (26) at (0.5, -0.5) {};
                \node [style=none] (27) at (1.25, -0.5) {};
                \node [style=none] (28) at (3.75, -0.5) {};
                \node [style=none] (29) at (0.5, -1.75) {};
                \node [style=none] (30) at (-2.75, 2.75) {kids};
                \node [style=none] (31) at (0.5, 2.75) {play};
                \node [style=none] (32) at (3.75, 2.75) {football};
                \node [style=state] (33) at (0.5, -1.75) {$\vec{w}$};
                \node [style=none] (34) at (0.6, -3.25) {politics};
        \end{pgfonlayer}
        \begin{pgfonlayer}{edgelayer}
                \draw [style = thick] (0.center) to (1.center);
                \draw [style = thick] (1.center) to (2.center);
                \draw [style = thick] (2.center) to (0.center);
                \draw [style = thick]  (3.center) to (4.center);
                \draw [style = thick] (8.center) to (6.center);
                \draw [style = thick] (6.center) to (7.center);
                \draw [style = thick] (4.center) to (5.center);
                \draw [style = thick] (3.center) to (5.center);
                \draw [style = thick] (8.center) to (7.center);
                \draw [style = thick] (9.center) to (14.center);
                \draw [style = thick] (10.center) to (15.center);
                \draw [style = thick] (11.center) to (16.center);
                \draw [style = thick] (12.center) to (17.center);
                \draw [style = thick] (13.center) to (18.center);
                \draw [thick, bend right=90, looseness=1.25] (24.center) to (25.center);
                \draw [thick, bend right=90, looseness=1.25] (27.center) to (28.center);
                \draw [style = thick] (26.center) to (29.center);
        \end{pgfonlayer}
\end{tikzpicture}
\end{aligned}
\end{equation}  
and assigning the sentence to the one of smaller value.  We would proceed with two steps:

\begin{enumerate}
\item Compute the derivation of $\ket{\psi}$, which, by classical direct calculation, takes  $\mathcal{O}(3N)$ operations. 

\item See which of $\vec{v}$ and $\vec{w}$ is closest to $\ket{\phi}$. This is an instance of the closest vector problem where $\vec{s} = \ket{\phi}$, $M=2$, and $U = \{\vec{v},\vec{w}\}$. With direct calculation or Monte Carlo the second step requires\footnote{If we assume the appropriate $d$-sparsity scaling.} $\mathcal{O}(2N)$ to be compared with the quantum method at $\mathcal{O}(\sqrt{2N}\log 2)$. Even if we include the step to import the classical data from step one into quantum form, which can be done with $\mathcal{O}(\log_2N)$ overhead \cite{giovannetti2008quantum}, then we obtain a speedup for this step. 
\end{enumerate}

Still, despite the quantum speedup from step two, the full algorithm for general $M$ runs in $\mathcal{O}(3N\sqrt{M}\log M)$, remaining linear in $N$.

In order to recover a speedup we refine the application of the quantum algorithm by posing a version of the closest vector problem that avoids the initial calculation of $\ket{\phi}$ altogether. Note the equivalence of the calculations in Equation \ref{eq:exampleCalcs} with 
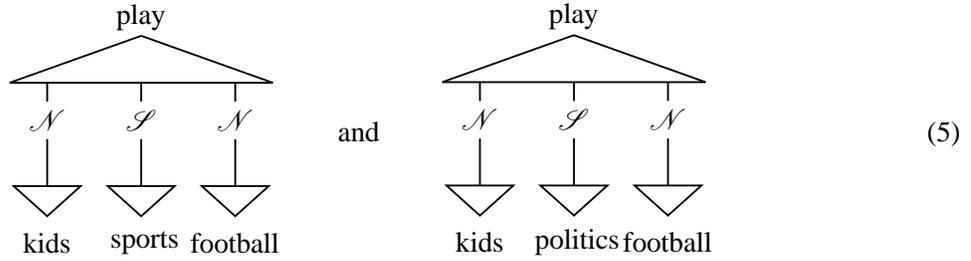
\begin{equation}
\label{eqn:trick}
\begin{aligned}
\begin{tikzpicture}[scale=0.5]
        \begin{pgfonlayer}{nodelayer}
                \node [style=none] (3) at (-3, 1) {};
                \node [style=none] (4) at (4, 1) {};
                \node [style=none] (5) at (0.5, 2.25) {};
                \node [style=none] (10) at (-2, 1) {};
                \node [style=none] (11) at (0.5, 1) {};
                \node [style=none] (12) at (3, 1) {};
                \node [style=none] (13) at (3.75, 1) {};
                \node [style=none] (14) at (-2.75, 0.5) {};
                \node [style=none] (15) at (-2, 0.5) {};
                \node [style=none] (16) at (0.5, 0.5) {};
                \node [style=none] (17) at (3, 0.5) {};
                \node [style=none] (18) at (3.75, 0.5) {};
                \node [style=none] (20) at (-2, 0) {$\mathcal{N}$};
                \node [style=none] (21) at (0.5, 0) {$\mathcal{S}$};
                \node [style=none] (22) at (3, 0) {$\mathcal{N}$};
                \node [style=none] (24) at (-2.75, -0.5) {};
                \node [style=none] (25) at (-0.25, -0.5) {};
                \node [style=none] (26) at (0.5, -0.5) {};
                \node [style=none] (27) at (1.25, -0.5) {};
                \node [style=none] (28) at (3.75, -0.5) {};
                \node [style=none] (29) at (0.5, -1.75) {};
                \node [style=none] (30) at (-2, -3.25) {kids};
                \node [style=none] (31) at (0.5, 2.75) {play};
                \node [style=none] (32) at (3, -3.25) {football};
                \node [style=state] (33) at (-2, -1.75) {};
                \node [style=state] (35) at (0.5, -1.75) {};
                \node [style=state] (36) at (3, -1.75) {};
                \node [style=none] (34) at (0.6, -3.25) {sports};
        \end{pgfonlayer}
        \begin{pgfonlayer}{edgelayer}
                \draw [style = thick]  (3.center) to (4.center);
                \draw [style = thick] (4.center) to (5.center);
                \draw [style = thick] (3.center) to (5.center);
                \draw [style = thick] (10.center) to (15.center);
                \draw [style = thick] (11.center) to (16.center);
                \draw [style = thick] (12.center) to (17.center);
                \draw [style = thick] (26.center) to (29.center);
                \draw [style = thick] (-2,-0.5) to (33.center);
                \draw [style = thick] (3,-0.5) to (36.center);
        \end{pgfonlayer}
\end{tikzpicture}
\end{aligned}
 \qquad \mbox{and} \qquad
\begin{aligned}
\begin{tikzpicture}[scale=0.5]
        \begin{pgfonlayer}{nodelayer}
                \node [style=none] (3) at (-3, 1) {};
                \node [style=none] (4) at (4, 1) {};
                \node [style=none] (5) at (0.5, 2.25) {};  
                \node [style=none] (10) at (-2, 1) {};
                \node [style=none] (11) at (0.5, 1) {};
                \node [style=none] (12) at (3, 1) {};
                \node [style=none] (13) at (3.75, 1) {};
                \node [style=none] (14) at (-2.75, 0.5) {};
                \node [style=none] (15) at (-2, 0.5) {};
                \node [style=none] (16) at (0.5, 0.5) {};
                \node [style=none] (17) at (3, 0.5) {};
                \node [style=none] (18) at (3.75, 0.5) {};
                \node [style=none] (20) at (-2, 0) {$\mathcal{N}$};
                \node [style=none] (21) at (0.5, 0) {$\mathcal{S}$};
                \node [style=none] (22) at (3, 0) {$\mathcal{N}$};
                \node [style=none] (24) at (-2.75, -0.5) {};
                \node [style=none] (25) at (-0.25, -0.5) {};
                \node [style=none] (26) at (0.5, -0.5) {};
                \node [style=none] (27) at (1.25, -0.5) {};
                \node [style=none] (28) at (3.75, -0.5) {};
                \node [style=none] (29) at (0.5, -1.75) {};
                \node [style=none] (30) at (-2, -3.25) {kids};
                \node [style=none] (31) at (0.5, 2.75) {play};
                \node [style=none] (32) at (3, -3.25) {football};
                \node [style=state] (33) at (-2, -1.75) {};
                \node [style=state] (35) at (0.5, -1.75) {};
                \node [style=state] (36) at (3, -1.75) {};
                \node [style=none] (34) at (0.6, -3.25) {politics};
        \end{pgfonlayer}
        \begin{pgfonlayer}{edgelayer}
                \draw [style = thick]  (3.center) to (4.center);
                \draw [style = thick] (4.center) to (5.center);
                \draw [style = thick] (3.center) to (5.center);
                \draw [style = thick] (10.center) to (15.center);
                \draw [style = thick] (11.center) to (16.center);
                \draw [style = thick] (12.center) to (17.center);
                \draw [style = thick] (26.center) to (29.center);
                \draw [style = thick] (-2,-0.5) to (33.center);
                \draw [style = thick] (3,-0.5) to (36.center);
        \end{pgfonlayer}
\end{tikzpicture}
\end{aligned}  
\end{equation}
Rather than directly calculating $\ket{\phi}$, which is not relevant to the classification task, we can formulate a closest vector problem where $\vec{s} = \ket{play}$, $M=2$ and $U = \{\ket{kids}\otimes\ket{v} \otimes\ket{football},\ket{kids}\otimes\ket{w} \otimes\ket{football}\}$.
The runtime of this \emph{deferred quantum algorithm}, including import, will then be $\mathcal{O}(\sqrt{MN})$, showing our desired quadratic speedup in both variable.

We extend this to result to include general sentences in the CSC model with the following theorem.

\begin{theorem}
For an $N$-dimensional noun meaning space, there exists a quantum algorithm to classify any CSC model sentence composed of $n$ tensors $\vec{w_0},\vec{w_1},...,\vec{w_{n-1}}$ into $M$ classes with time $\mathcal{O}(\sqrt{MN}\log M)$. This improves on classical methods that require $\mathcal{O}(NM)$.
\end{theorem}
\begin{proof}
The trick we play in Equation \ref{eqn:trick} amounts to splitting the sentence derivation into a bipartite graph.  As the CSC connections are based on a pregroup derivation, the connections will always form a tree, taking words as nodes and connections as edges. Trees can always be partitioned into bipartite graphs, thus, up to the ordering of inputs on each tensor which can be kept track of, we can always give a deferred quantum algorithm with associated speedup for any such CSC sentence.
 The following procedure explicitly details how to construct this bipartite partitioning.

For every CSC sentence there is one word that acts as the \emph{head} of the derivation.  This is the word whose output $S$ wire contains the sentence meaning following its derivation's linear map. In Equation \ref{eqn:phi} this is the word ``play". Connect the dangling wire of the head word $\vec{w_h}$ with the vector $\vec{v_i}$ against which similarity is being measured.  Starting with this head word we then separate the sentence into a top layer and a bottom layer with the following steps.  Assign the head word to the top layer. Place every word it is connected to on the bottom layer. Next for every word just assigned to the bottom, take all their connected words, which are not yet assigned, and assign them to the top.  Continue this procedure while alternating top and bottom until all words are assigned. This gives a simple two-coloring of the derivation graph. 
\end{proof}

\begin{example}
Consider the following sentence \cite{dimitri}:
\[
\begin{tikzpicture}[scale=1]
                \node [style=state, hflip] (0) at (-6, 1) {1};
                \node [style=none] (1) at (-5.25, 0.25) {};
                \node [style=state, hflip] (2) at (-2, 1) {2};
                \node [style=state, hflip] (3) at (2, 1) {2};
                \node [style=state, hflip] (4) at (3.5, 1) {3};
                \node [style=none] (5) at (-1.25, 0.25) {};
                \node [style=none] (6) at (1.25, 0.25) {};
                \node [style=none] (7) at (2.75, 0.25) {};
                \node [style=none] (8) at (-4.5, 1) {};
                \node [style=none] (9) at (-3.5, 1) {};
                \node [style=none] (10) at (-0.5, 1) {};
                \node [style=none] (11) at (0.5, 1) {};
                \node [style=none] (label1) at (0, 1.175) {1};
                \node [style=state, hflip, xscale=2] (12) at (0, 1) {};
                \node [style=none] (label1) at (-4, 1.175) {0};
                \node [style=state, hflip, xscale=2] (13) at (-4, 1) {};
                \node [style=none] (14) at (-4, -1) {};
                \node [style=none] (15) at (-1.75, -0.5) {};
                \node (16) at (-6, 1.75) {John};
                \node (17) at (-4, 1.75) {saw};
                \node (18) at (-2, 1.75) {Mary};
                \node (19) at (0, 1.75) {read};
                \node (20) at (2, 1.75) {a};
                \node (21) at (3.5, 1.75) {book};
                \draw [bend right=45, looseness=1.00] (0) to (1.center);
                \draw [bend right=45, looseness=1.00] (2) to (5.center);
                \draw [in=-90, out=0, looseness=1.00] (7.center) to (4);
                \draw [bend left=45, looseness=1.00] (8.center) to (1.center);
                \draw (13) to (14.center);
                \draw [bend left=45, looseness=1.00] (10.center) to (5.center);
                \draw [bend right=45, looseness=1.00] (9.center) to (15.center);
                \draw [bend right=45, looseness=1.00] (11.center) to (6.center);
                \draw [bend right=45, looseness=1.00] (6.center) to (1.9,1);
                \draw [bend right=45, looseness=1.00] (2.1,1) to (7.center);
                \draw [bend right=45, looseness=1.00] (15.center) to (12);
\end{tikzpicture}
\]
where we have labeled the vectors based on their depth in the derivation tree.  The two-layer form assigns even vectors to the top layer and odd vectors to the bottom:
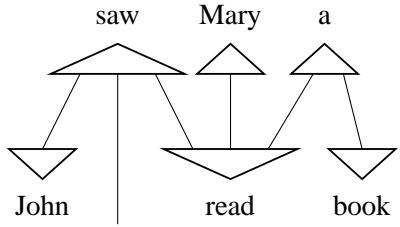
\[
\begin{tikzpicture}[scale=1]
                \node [style=state] (0) at (-5, 0) {};
                \node [style=state, hflip] (1) at (-2.5, 1) {};
                \node [style=state, hflip] (2) at (-1.25, 1) {};
                \node [style=state] (3) at (-0.75, 0) {};
                \node (4) at (-4.5, 1) {};
                \node (5) at (-3.5, 1) {};
                \node (6) at (-2.5, -1) {};
                \node (7) at (-2, 0) {};
                \node [style=state, xscale=2] (8) at (-2.5, 0) {};
                \node [style=state, hflip, xscale=2] (9) at (-4, 1) {};
                \node (10) at (-4, -1) {};
                \node (11) at (-5, -0.75) {John};
                \node (12) at (-4, 1.75) {saw};
                \node (13) at (-2.5, 1.75) {Mary};
                \node (14) at (-2.5, -0.75) {read};
                \node (15) at (-1.25, 1.75) {a};
                \node (16) at (-0.75, -0.75) {book};
                \draw (9) to (10.center);
                \draw (4.center) to (0);
                \draw (1) to (-2.5,0);
                \draw (7.center) to (-1.4,1);
                \draw (-1,1) to (3);
                \draw (5.center) to (-3,0);   
\end{tikzpicture}
\]
\end{example} 

\noindent Hooking the dangling wire up to a classifying vector reduces the similiarity computation to the calculation of a single inner product. Note that this procedure works for any derivation tree,\footnote{Even non-pregroup and non-CCG models will work as long as there is some tree derivation.} so sentence fragments, such as noun phrases, can be easily analyzed in exactly the same manner.

\section{Noise tolerance and Conclusion}  

To reap the benefits of quantum algorithm in the domain of natural language processing, we  apply a new technique to defers the calculation of a sentences compositional meaning, eliminating the overhead costs. By this method we are able to introduce a quantum algorithm for calculating sentence similarity that offers quadratic speedup over classical direct calculation and Monte-Carlo methods. These kinds of algorithms are particularly attractive for practical applications of quantum computing as noisy results can be tolerated: in our case when the desired errors is lower bounded by $1/\sqrt{N}$.  Vector space models are already inherently noisy and typical tasks allow for errors in results, so this restriction does not affect rule out their efficacy. 

An additional point is that the density matrix formalism of \cite{piedeleu2015open, bankova2016graded} can also be naturally modeled by mixed states of quantum systems.  In fact, this analogy was the genesis for the theory of disambiguation presented there, as another example of the shared structure that led to the results presented here. At a basic level, our work exploits the abstract connection between natural language processing and quantum information.  More formally,  we can see both quantum computation in the category of finite dimensional Hilbert spaces and linear maps \cite{abramsky2004categorical,qcs-notes} and CSC in the product category of pregroup grammar and finite dimensional vectors spaces \cite{coecke2010mathematical}. The connection between these two (as dagger compact categories) makes the application of one to the other apparent.      
   
%\section{Acknowledgements}
%
%The authors would especially like to thank the advice in practical natural language processing from Dimitri Kartsaklis, Stephen Clark, Mehrnoosh Sadrzadeh, and Tamara Polajnar as well as others involved with the DisCo project. Diagrams were created with the TikZ package. 


%+Bibliography%
\bibliographystyle{eptcs}
\providecommand{\urlalt}[2]{\href{#1}{#2}}
\bibliography{NLP}
%-Bibliography

\end{document}